\documentclass[twoside,11pt]{article}
\usepackage{times}
\usepackage{amsmath,amssymb}
\usepackage{graphicx}
\usepackage{subfigure}
\usepackage{multirow}
\usepackage{jmlr2e}
\usepackage[dvipdfm]{hyperref}

\makeatletter
\newcommand{\figcaption}[1]{\def\@captype{figure}\caption{#1}}
\newcommand{\tblcaption}[1]{\def\@captype{table}\caption{#1}}
\makeatother

\newcommand{\Real}{{\mathbb{R}}}

\newcommand{\minimize}{\mathop{\mathrm{minimize}}}

\newcommand{\Eqref}[1]{Equation~{\eqref{#1}}}
\newcommand{\Secref}[1]{Section~{\ref{#1}}}
\newcommand{\Figref}[1]{Figure~{\ref{#1}}}

\newcommand{\T}{^\top}

\def\dot#1#2{\left\langle #1,#2 \right\rangle}
\allowdisplaybreaks

\newcommand{\bm}[1]{\boldsymbol{#1}}

\newcommand{\vd}{\bm{d}}
\newcommand{\vf}{\bm{f}}
\newcommand{\vg}{\bm{g}}
\newcommand{\vw}{\bm{w}}
\newcommand{\vx}{\bm{x}}
\newcommand{\vy}{\bm{y}}

\newcommand{\valpha}{\bm{\alpha}}
\newcommand{\veta}{\bm{\eta}}

\newcommand{\mI}{\bm{I}}
\newcommand{\mK}{\bm{K}}

\newcommand{\gf}{\mathfrak{g}}

\newcommand{\calH}{{\mathcal H}}
\newcommand{\calX}{{\mathcal X}}
\newcommand{\calY}{{\mathcal Y}}

\newcommand{\argmin}{\mathop{\mathrm{argmin}}}

\title{Regularization Strategies and Empirical Bayesian Learning for MKL}
\author{\name Ryota Tomioka \email tomioka@mist.i.u-tokyo.ac.jp \\
\name Taiji Suzuki \email s-taiji@stat.t.u-tokyo.ac.jp\\
\addr Department of Mathematical Informatics, The University
of Tokyo,\\ 7-3-1, Hongo, Bunkyo-ku, Tokyo, 113-8656, Japan.
}
\editor{S\"oren Sonnenburg, Francis Bach, and Cheng Soon Ong}

\begin{document}
\maketitle

\begin{abstract}
Multiple kernel learning (MKL), structured sparsity, and multi-task
learning have recently received considerable attention. In this paper,
we show how different MKL algorithms can be understood as applications
of either regularization on the kernel weights or block-norm-based
regularization, which is more common in structured sparsity and
 multi-task learning. We show that these two regularization strategies can be
systematically mapped to each other through a concave conjugate
 operation. When the  kernel-weight-based regularizer is separable into
 components, we can naturally consider 
 a generative probabilistic model behind MKL. Based on this model, we
 propose learning algorithms for the kernel weights through the
 maximization of marginal likelihood. We show through numerical
 experiments that $\ell_2$-norm MKL and Elastic-net MKL achieve
 comparable accuracy to uniform kernel combination. Although uniform kernel
 combination might be preferable from its simplicity, $\ell_2$-norm MKL
 and Elastic-net MKL can learn the usefulness of the information sources
 represented as kernels. In particular, Elastic-net MKL achieves
 sparsity in the kernel weights.
\end{abstract}

\section{Introduction}
In many learning problems, the choice of feature representation,
descriptors, or kernels plays a crucial role. 
The optimal representation is problem specific. For example, we can
represent a web page as a bag-of-words, which might help us in
classifying whether the page is discussing politics or economy; we can
also represent the same page by the links provided in the page, which
could be more useful in classifying whether the page is supporting political
party A or B. Similarly in a visual categorization task, a color-based
descriptor might be useful in classifying an apple from a lemon but not in discriminating an airplane from a car.
Given that there is no single feature representation that works in every
learning problem, it is crucial to combine them in a problem dependent manner
for a successful data analysis.

In this paper, we consider the problem of combining multiple
data sources in a kernel-based learning framework.
%
More specifically, we assume that a data point $x\in\calX$ lies in
a space $\mathcal{X}$ and we are given $M$ candidate kernel functions
$k_m:\calX\times\calX\rightarrow\Real$  ($m=1,\ldots,M$). Each kernel
function corresponds to one data source. A conical combination of $k_m$
($m=1,\ldots,M$) gives the combined kernel function
$\bar{k}=\sum_{m=1}^Md_mk_m$, where $d_m$ is a nonnegative weight. 
Our goal is to find a good set of kernel weights based on some
training examples.


Various approaches have been proposed for the above problem under the name
multiple kernel learning
(MKL)~\citep{JMLR:Lanckriet+etal:2004,ICML:Bach+etal:2004,ZieOng07,VarRay07,AflBenBhaNatRam09,GehNow09,KloBreSonLasMueZie09,LonGal09,UAI:Cortes+etal:2009}. 
Recently, \cite{KloBreSonLasMueZie09,KloRucBar10b} have shown that
many MKL approaches can be understood as application of the
penalty-based regularization (Tikhonov regularization) or
constraint-based regularization (Ivanov regularization) on the kernel
weights $d_m$. Meanwhile,
there is a growing interest in learning under {\em structured
sparsity}
assumption~\citep{JRSS:YuanLin:2006,ArgEvgPon08,HuaZhaMet09,JacOboVer09,JenAudBac09},
which employs another regularization based on the so-called {\em block-norm}.

Therefore, a natural question to ask is how these two regularization
strategies are related to each other. For simple cases the
correspondence is well known (see
\cite{ICML:Bach+etal:2004,BacThiJor05,KloBreSonLasMueZie09}). Moreover,
in the context of structured sparsity, \cite{MicMorPon10} have proposed
a more sophisticated class of penalty functions that employ both
Tikhonov and Ivanov regularizations on the kernel weights and have shown
the corresponding block-norm regularization.  The first contribution 
of this paper is to show that under some mild assumptions the kernel-weight-based
regularization and the block-norm-based regularization can be mapped to
each other in a systematic manner through a concave conjugate operation.

All the regularization strategies we discussed so far is formulated as 
convex optimization problems. The second contribution of this paper is to
propose a nonconvex regularizer based on the
marginal likelihood. Although the overall minimization problem is
nonconvex, we propose a iterative algorithm that alternately
 minimize two convex objectives.
 Although Bayesian approaches have been applied to MKL
earlier in a transductive nonparametric setting by \cite{ZhaYeuKwo04}, and
a setting similar to the relevance vector machine~\citep{Tip01} by
\cite{GirRog05,DamGir08}, our formulation is more
coherent with the correspondence between Gaussian process
classification/regression and kernel methods~\citep{RasWil06}. Note that
very recently \cite{ArcBac10,Urt10} also studied similar models.

This paper is structured as follows. In \Secref{sec:mkl}, we start from 
analyzing learning with fixed kernel combination. Then we discuss the
two regularization strategies (kernel-weight-based regularization and
block-norm-based regularization). Finally, we present our main result on
the correspondence between the two formulations. 
In \Secref{sec:bayes}, we start from viewing the separable kernel-weight-based model  as a
hierarchical  maximum a posteriori(MAP) estimation problem, and propose 
an empirical Bayesian approach for the same model. Furthermore, we show
a connection to the general framework we discuss in \Secref{sec:mkl}.
We numerically compare the proposed empirical Bayesian MKL and various
MKL models  on visual categorization tasks from the Caltech
101 dataset~\citep{FeiFerPer04} using 1,760 kernel functions. Finally, we summarize our contributions in \Secref{sec:conclusion}.

\section{Multiple kernel learning frameworks and their connections}
\label{sec:mkl}
In this section, we first consider the problem of learning a classifier
with a fixed kernel combination. Then we extend this framework to jointly
optimize the kernel weights together with the classifier. Second, we
consider the {\em block-norm based regularization}, which have been
discussed in structured sparsity literature (including group lasso)~\citep{JRSS:YuanLin:2006,ArgEvgPon08,HuaZhaMet09,JacOboVer09,JenAudBac09}. 
Our main concern is in how these two formulations are related to each
other. We show two theorems (Theorem \ref{thm:forward} and
\ref{thm:backward}) that map the two formulations. Using these theorems, we show that previously
proposed regularized MKL models can be systematically transformed from
one formulation to another.

\subsection{Learning with fixed kernel combination}
\label{sec:fixed}
We assume that we are given $N$ training examples
$
(x_i,y_i)_{i=1}^{N}
$
where $x_i$ belongs to an input space $\calX$ and $y_i$ belongs to an
output space $\calY$ (usual settings are $\calY = \{\pm 1\}$ for
classification and $\calY = \Real$ for regression).

We first consider a learning problem with fixed kernel weights. More
specifically, we fix non-negative kernel weights $d_1,d_2,\ldots,d_M$ 
and consider the RKHS $\bar{\calH}$ corresponding to the
combined kernel function $\bar{k}=\sum_{m=1}^Md_mk_m$.
 The squared RKHS norm of a function
$\bar{f}$ in the combined RKHS $\bar{\calH}$ can be represented as follows:
\begin{align}
\label{eq:mkl-norm}
 \|\bar{f}\|_{\bar{\calH}}^2:=\min_{\substack{f_1\in\calH_1,\\\ldots,f_M\in\calH_M}}\sum_{m=1}^{M}\frac{\|f_m\|_{\calH_m}^2}{d_m}\quad
 {\rm s.t.}\,\, \bar{f}=\sum_{m=1}^Mf_m,
\end{align}
where $\calH_m$ is the RKHS that corresponds to the kernel
function $k_m$.  If $d_m=0$, the ratio $\|f_m\|_{\calH_m}^2/d_m$ is
defined to be zero if $\|f_m\|_{\calH_m}=0$ and infinity otherwise.
See  Sec 6 in \cite{AMS:Aronszajn:1950}, and also
Lemma~25 in \cite{JMLR:MicchelliPontil:2005} for the proof. We also provide some
intuition for a finite dimensional case in
Appendix~\ref{sec:proof-mkl-norm}. 


Using the above representation, a supervised learning
problem  with a fixed kernel combination can be written as follows:
\begin{align}
\label{eq:fixed-kernel-learning}
 \minimize_{
\substack{f_1\in\calH_1,\\\ldots,
 f_M\in\calH_M,\\
b\in\Real}}&\sum_{i=1}^N\ell\left(y_i,\textstyle\sum_{m=1}^Mf_m(x_i)+b\right)
+\frac{C}{2}\sum_{m=1}^M\frac{\|f_m\|_{\calH_m}^2}{d_m},
\end{align}
where $\ell:\Real\times\Real\rightarrow\Real$ is a loss function and we
assume that $\ell$ is convex in the second argument; for example, the
loss function can be the hinge loss 
$\ell_H(y_i,z_i)=\max(0,1-y_iz_i)$,
or the quadratic loss
$\ell_Q(y_i,z_i)=(y_i-z_i)^2/(2\sigma_y^2)$.

It might seem that we are making the problem unnecessarily complex by
introducing $M$ functions $f_m$ to optimize instead of simply optimizing over
$\bar{f}$. However, explicitly handling the kernel weights enables us to
consider various regularization strategies on the weights as we see in
the next subsection.

%

\subsection{Kernel-weight-based regularization}
\label{sec:kernel-weight}
Now we are ready to also optimize the kernel weights $d_m$ in the
above formulation.
Clearly there is a need for regularization, because the
objective~\eqref{eq:fixed-kernel-learning} is a monotone decreasing
function of the kernel weights $d_m$. Intuitively speaking, $d_m$
corresponds to the complexity allowed for the $m$th regression function
$f_m$; the more complexity we allow, the better the fit to the training
examples becomes. Thus without any constraint on $d_m$, we can get a
severe overfitting problem.

Let $h:\Real^M_+\rightarrow \Real\cup\{+\infty\}$ be a function from
a non-negative real vector $\vd\in\Real_+^{M}$ to a real number.
One way to penalize the complexity is to minimize the
objective~\eqref{eq:fixed-kernel-learning} together with the regularizer
$h(\vd)$ as follows:
\begin{align}
\label{eq:mkl-tikhonov}
  \minimize_{
\substack{
f_1\in\calH_1,\ldots,f_M\in\calH_M,\\
b\in\Real,\\
\vd\in\Real_+^{M}}}&\sum_{i=1}^N\ell\left(y_i,\textstyle\sum_{m=1}^Mf_m(x_i)+b\right)
+\frac{C}{2}\left(\sum_{m=1}^M\frac{\|f_m\|_{\calH_m}^2}{d_m}+  h(\vd)\right).
\end{align}
Note that if $h$ is a convex function, the above optimization problem is
{\em jointly convex} in $f_m$ and $d_m$ (see \citet[Example 3.18]{book:Boyd+Vandenberghe:2004})

\begin{example}[$\ell_p$-norm MKL via Tikhonov regularization]
\label{ex:lp-tikhonov}
Let $h(\vd)=\sum_{m=1}^Md_m^p/p$. Then we have
\begin{align*}
 \minimize_{
\substack{
f_1\in\calH_1,\ldots,f_M\in\calH_M,\\
b\in\Real,\\
\vd\in\Real_+^{M}}}&\sum_{i=1}^N\ell\left(y_i,\textstyle\sum_{m=1}^Mf_m(x_i)+b\right)
+\frac{C}{2}\sum_{m=1}^M\left(\frac{\|f_m\|_{\calH_m}^2}{d_m}+\frac{d_m^p}{p} \right)
\end{align*}
\end{example}
The special case $p=1$ was considered earlier in \cite{VarRay07}.

\begin{example}[$\ell_p$-norm MKL via Ivanov regularization]
\label{ex:lp-ivanov} 
Let $h(\vd)=0$ if $\sum_{m=1}^Md_m^p\leq 1$ and $h(\vd)=+\infty$
 otherwise. Then we have
\begin{align*}
 \minimize_{
\substack{
f_1\in\calH_1,\ldots,f_M\in\calH_M,\\
b\in\Real,\\
\vd\in\Real_+^{M}}}&\sum_{i=1}^N\ell\left(y_i,\textstyle\sum_{m=1}^Mf_m(x_i)+b\right)
+\frac{C}{2}\sum_{m=1}^M\frac{\|f_m\|_{\calH_m}^2}{d_m}\quad {\rm
 s.t.}\quad \sum_{m=1}^{M}d_m^p\leq 1.
\end{align*}
\end{example}
This formulation was considered by
\cite{KloBreLasSon08,KloBreSonLasMueZie09}. The special case $p=2$ was
considered by \cite{UAI:Cortes+etal:2009}.

\begin{example}[Multi-task learning]
\label{ex:multi-task} 
In this example, the linear combination inside the loss term is defined
 in a sample-dependent manner. More precisely, we assume that there are
 $n$ tasks, and each sample is associated with the $l(i)$th
 task. Let $\calH$ be an RKHS over the input space $\calX$. We consider
 $M=n+1$ functions $f_1,\ldots,f_{M}\in\calH$ to model the task
 dependent component and the task independent component.
 The first $n$ functions
 $f_1,\ldots,f_n\in\calH$  represent the task dependent components of
 the classifiers, and the last function $f_{M}\in\calH$ represents the
component that is common to all tasks. Accordingly, the optimization
problem can be expressed as follows:
\begin{align*}
\minimize_{
\substack{
f_1\in\calH_1,\ldots,f_M\in\calH_M,\\
b\in\Real,\\
\vd\in\Real_+^{M}}}&\sum_{i=1}^N\ell\left(y_i, f_{l(i)}(x_i)+f_{M}(x_i)+b\right)
+\frac{C}{2}\sum_{m=1}^M\frac{\|f_m\|_{\calH_m}^2}{d_m}\quad {\rm
 s.t.}\quad \sum_{m=1}^{M}d_m\leq 1.
\end{align*}
\end{example}
\cite{EvgPon04,EvgMicPon05} proposed a related model that only has one
 parameter; they used 
 $d_m=n\lambda$ for $m=1,\ldots,n$ and $\sum_{m=1}^{n}d_m/n^2+d_M\leq 1$
 instead of the constraint in the above formulation. However they did
 not discuss  joint optimization of the hyperparameter $\lambda$ and the
 classifier. 

\begin{example}[Wedge penalty]
\label{ex:wedge}
Let $h(\vd)=\sum_{m=1}^{M}d_m$ if $d_m\geq d_{m+1}$ for all $m=1,\ldots,M-1$, and $h(\vd)=+\infty$
 otherwise. Then we have
\begin{align*}
\minimize_{
\substack{
f_1\in\calH_1,\ldots,f_M\in\calH_M,\\
b\in\Real,\\
\vd\in\Real_+^{M}}}&\sum_{i=1}^N\ell\left(y_i,\textstyle\sum_{m=1}^Mf_m(x_i)+b\right)
+\frac{C}{2}\sum_{m=1}^M\left(\frac{\|f_m\|_{\calH_m}^2}{d_m}+d_m\right)\quad {\rm
 s.t.}\quad \vd\in W,
\intertext{where}
W&=\{\vd: \vd\in\Real_+^M, d_m\geq d_{m+1}, m=1,\ldots,M-1\}.
\end{align*}
\end{example}
This penalty function was considered by \cite{MicMorPon10}.

\subsection{Block-norm-based regularization}
\label{sec:block-norm}
Historically \cite{ICML:Bach+etal:2004,BacThiJor05} and
\cite{JMLR:MicchelliPontil:2005} pointed out
that the problem of learning kernel weights and classifier
simultaneously can be reduced to the problem of learning the classifier
under some special regularizer, which we call {\em block norm based
regularizer} in this paper. Generalizing the presentation in the earlier
papers, we define the block-norm-based regularization as follows:
\begin{align}
\label{eq:mkl-block-gen}
  \minimize_{f_1\in\calH_1,\ldots,
 f_M\in\calH_M,b\in\Real}&\sum_{i=1}^N\ell\left(y_i,\textstyle\sum_{m=1}^Mf_m(x_i)+b\right)+Cg(\|f_1\|_{\calH_1}^2,\ldots,\|f_M\|_{\calH_M}^2),
\end{align}
where $g:\Real_+^{M}\rightarrow \Real$ is called the block-norm-based regularizer.

\begin{example}[Block 1-norm MKL]
\label{ex:block1}
Let $g(\vx)=\sum_{m=1}^{M}\sqrt{x_m}$. Then we have
\begin{align}
 \label{eq:mkl-block1}
 \minimize_{f_1\in\calH_1,\ldots,
 f_M\in\calH_M,b\in\Real}&\sum_{i=1}^N\ell\left(y_i,\textstyle\sum_{m=1}^Mf_m(x_i)+b\right)
+C\sum\nolimits_{m=1}^M\|f_m\|_{\calH_m}.
\end{align}
\end{example}
This formulation was discussed earlier in \cite{BacThiJor05}. When all
the kernels are linear kernels defined on non-overlapping subset of
input variables, this is equivalent to the group lasso~\cite{JRSS:YuanLin:2006}.

\begin{example}[Overlapped group lasso]
Suppose that we have $D$ input variables $\vx=(x^{(1)},\ldots,x^{(D)})\T$
 and kernel functions $k_m$ are defined as overlapped linear kernels as
 follows:
\begin{align}
\label{eq:kern-overlasso}
 k_m(\vx,\vy)&=\sum_{l\in \gf_m}x^{(l)}y^{(l)}\quad (m=1,\ldots,M),
\end{align}
where $\gf_m$ $(m=1,\ldots,M)$ is a subset of (overlapped) indices from $\{1,\ldots,D\}$.
Introducing weight vectors $\vw_m\in\Real^{|\gf_m|}$ $(m=1,\ldots,M)$,
 we can rewrite $f_m(\vx)=\vw_m\T\vx^{(\gf_m)}$, where
 $\vx^{(\gf_m)}=(x^{(l)})_{l\in\gf_m}\T$. In addition,
 $\|f_m\|_{\calH_m}^2=\|\vw_{m}\|^2$. 
Then employing the same regularizer as in Example~\ref{ex:block1}, we
 have
\begin{align}
 \label{eq:overlasso}
 \minimize_{\vw_1\in\Real^{|\gf_1|},\ldots,
 \vw_M\in\Real^{|\gf_M|},b\in\Real}&\sum_{i=1}^N\ell\left(y_i,\textstyle\sum_{m=1}^M\vw_m\T\vx_i^{(\gf_m)})+b\right)
+C\sum\nolimits_{m=1}^M\|\vw_m\|.
\end{align}
\end{example}
This formulation was considered by \cite{JacOboVer09}. Except for the
particular choice of the kernel function~\eqref{eq:kern-overlasso}, this
is a special case of the block 1-norm MKL in Example~\ref{ex:block1}.

\begin{example}[Elastic-net MKL]
\label{ex:elastic}
Let
 $g(\vx)=\sum_{m=1}^{M}\left((1-\lambda)\sqrt{x_m}+\frac{\lambda}{2}x_m\right)$. Then
 we have
\begin{align}
\label{eq:mkl-block-elastic}
  \minimize_{\substack{f_1\in\calH_1,\ldots,
 f_M\in\calH_M,\\b\in\Real}}&\sum_{i=1}^N\ell\left(y_i,\textstyle\sum_{m=1}^Mf_m(x_i)+b\right)+C\sum_{m=1}^M\left((1-\lambda)\|f_m\|_{\calH_m}+\frac{\lambda}{2}\|f_m\|_{\calH_m}^2\right).
\end{align}
\end{example}
This formulation was discussed earlier in \cite{Sha08,LonGal09,arXiv:SparsityTradeoff:2010}.
Note that the elastic-net MKL~\eqref{eq:mkl-block-elastic} reduces to
 the block 1-norm MKL~\eqref{eq:mkl-block1}) for $\lambda=0$ and the
uniform-weight combination ($d_m=1$
in \Eqref{eq:fixed-kernel-learning}) for $\lambda=1$.

\subsection{Connection between the two formulations}
In this subsection, we present two theorems that connect the
kernel-weight-based regularization (\Secref{sec:kernel-weight}) and the
block-norm-based regularization (\Secref{sec:block-norm}). 

The following theorem states that under some assumptions about the
regularizer $h$, we can {\em analytically} eliminate the kernel weights $\vd$ from the
optimization problem~\eqref{eq:mkl-tikhonov}.

\begin{theorem}
\label{thm:forward}
 We assume that the kernel weight based regularizer $h$ is convex, zero at the origin, and satisfies the generalized monotonicity in the following sense: for
$\vx,\vy\in\Real^M_+$ satisfying $x_m\leq y_m$ ($m=1,\ldots,M$), $h$ satisfies
\begin{align}
\label{eq:nondecreasing}
 h(\vx)\leq h(\vy).
\end{align}
Moreover, let $\tilde{h}(\vy):=-h(1/y_1,\ldots,1/y_M)$.
Then $\tilde{h}$ is a concave function and the optimization problem~\eqref{eq:mkl-tikhonov} can be reduced to
 the following one:
\begin{align}
 \label{eq:mkl-block}
\minimize_{\substack{
f_1\in\calH_1,\ldots,f_M\in\calH_M,\\
b\in\Real}}
&\sum_{i=1}^N\ell\left(y_i,\textstyle\sum_{m=1}^Mf_m(x_i)+b\right)+Cg(\|f_1\|_{\calH_1}^2,\ldots,\|f_M\|_{\calH_M}^2),
\intertext{where}
\label{eq:gdef}
g(\vx)&=\frac{1}{2}\inf_{\vy\in\Real_+^M}\left(\vx\T\vy-\tilde{h}(\vy)\right)
\end{align}
is the concave conjugate function of $\tilde{h}$ (divided by
two).
 Moreover, if $g$ is differentiable, the optimal kernel weight
$d_m$ is obtained as follows:
\begin{align*}
 d_m&=\left(2\frac{\partial g(\|f_1\|_{\calH_1}^2,\ldots,\|f_M\|_{\calH_M}^2)}{\partial x_m}\right)^{-1}.
\end{align*}
\end{theorem}
\begin{proof}
 In order to show the concavity of $\tilde{h}$, we show the convexity of
 $h(1/y_1,\ldots,1/y_M)$. This is a straightforward generalization of
 the scalar composition rule in
 \cite[p84]{book:Boyd+Vandenberghe:2004}. Let $\phi$ be any scalar
 convex function. Then
\begin{align*}
&h(\phi(\theta x_1+(1-\theta)y_1),\ldots,\phi(\theta x_M+(1-\theta)y_M))\\
&\leq h(\theta \phi(x_1)+(1-\theta)\phi(y_1),\ldots,\theta \phi(x_M)+(1-\theta)\phi(y_M))\\
&\leq \theta h(\phi(x_1),\ldots,\phi(x_M))+(1-\theta)h(\phi(y_1),\ldots,\phi(y_M)),
\end{align*}
In the second line, we used the convexity of $\phi$ and the
 monotonicity of $h$ in \Eqref{eq:nondecreasing}.
Therefore, letting $\phi(x)=1/x$, we have the convexity of
 $h(1/y_1,\ldots,1/y_M)$ and the concavity of $\tilde{h}$. Now if the
 infimum in the minimization \eqref{eq:gdef} exists in $\Real_+^M$, we
 have the block-norm formulation~\eqref{eq:mkl-block} by substituting
 $x_m=\|f_m\|_{\calH_m}^2$ and $y_m=1/d_m$ for $m=1,\ldots,M$. The case 
 $y_m=+\infty$ for some $m$ happens only if $x_m=\|f_m\|_{\calH_m}^2=0$
 because $\tilde{h}(\vy)\leq 0$, and this corresponds to $d_m=0$ in the
 original formulation~\eqref{eq:mkl-tikhonov}. 
The  last part of the theorem follows from fact that
 $\tilde{h}(\vy)=\inf_{\vx\in\Real_+^M}\left(\vx\T\vy-2g(\vx)\right)$
 and the optimality condition
\begin{align*}
 \frac{1}{d_m}-2\frac{\partial
g(\|f_1\|_{\calH_1}^2,\ldots,\|f_M\|_{\calH_M}^2)}{\partial x_m}=0.
\end{align*}
\end{proof}
Note that the monotonicity assumption~\eqref{eq:nondecreasing} and the
convexity of $h$ form a strong sufficient condition for the above
correspondence to hold. What we need is the concavity of $\tilde{h}$. In
the context of variational Bayesian inference, such condition has been
intensively studied. See \Secref{sec:bayes} and \cite{WipNag09,SeeNic08}.

In the particularly simple cases where the kernel-weight-based
regularizer $h$ is {\em separable}, the block-norm-based regularizer $g$
is also separable as in the following corollary.
\begin{corollary}[Separable regularizer]
\label{cor:separable}
Suppose that the kernel-weight-based regularizer is defined as a separable
 function (with a slight abuse of notation) as follows:
\begin{align*}
h_{\rm sep}(\vd)&=\sum_{m=1}^{M}h(d_m),
\end{align*}
and $h$ is convex and nondecreasing.
Then, the  corresponding block-norm-based regularizer is also separable and can be
 expressed as follows:
\begin{align*}
 g_{\rm sep}(\vx)&=\frac{1}{2}\sum_{m=1}^{M}\inf_{y_m\geq 0}\left(x_my_m-\tilde{h}(y_m)\right),
\end{align*}
where $\tilde{h}(y)=-h(1/y)$.
\end{corollary}
\begin{proof}
 The proof is straightforward and is omitted (see e.g., \citet[p95]{book:Boyd+Vandenberghe:2004}).
\end{proof}
Separable regularizers proposed earlier in literature are summarized in Table~\ref{tab:correspond}.

The monotonicity assumption~\eqref{eq:nondecreasing} holds for the
regularizers defined in
Examples~\ref{ex:lp-tikhonov}--\ref{ex:wedge}. Therefore, we have the
following corollaries.

\begin{corollary}[Block $q$-norm formulation via Tikhonov regularization]
Applying Theorem~\ref{thm:forward} to the $\ell_p$-norm MKL in
 Example~\ref{ex:lp-tikhonov}, we have
\begin{align}
\label{eq:mkl-blockq-1}
 \minimize_{f_1\in\calH_1,\ldots,
 f_M\in\calH_M,b\in\Real}&\sum_{i=1}^N\ell\left(y_i,\textstyle\sum_{m=1}^Mf_m(x_i)+b\right)
+\frac{C}{q}\sum\nolimits_{m=1}^M\|f_m\|_{\calH_m}^{q},
\end{align}
where  $q=2p/(1+p)$.
\end{corollary}
\begin{proof}
Note that the regularizer  $h(\vd)=\sum_{m=1}^{M}d_m^p/p$
in Example~\ref{ex:lp-tikhonov} is separable as in
 Corollary~\ref{cor:separable}. Therefore, we define
 $\tilde{h}(y_m)=-y_m^{-p}/p$, and accordingly we have
 $g(x_m)=\frac{1+p}{2p}x_m^{p/(1+p)}$, from which we obtain
 \Eqref{eq:mkl-blockq-1}.
\end{proof}

\begin{corollary}[Block $q$-norm formulation via Ivanov
 regularization]
\label{cor:blockq-ivanov}
 Applying Theorem~\ref{thm:forward} to the
 $\ell_p$-norm MKL in Example~\ref{ex:lp-ivanov}, we have
\begin{align}
 \label{eq:mkl-blockq-2}&
 \minimize_{f_1\in\calH_1,\ldots,
 f_M\in\calH_M,b\in\Real}&\sum_{i=1}^N\ell\left(y_i,\textstyle\sum_{m=1}^Mf_m(x_i)+b\right)
+\frac{C}{2}\left(\sum\nolimits_{m=1}^M\|f_m\|_{\calH_m}^{q}\right)^{2/q},
\end{align}
 where $q=2p/(1+p)$.
\end{corollary}
\begin{proof}
 For the regularizer defined in Example~\ref{ex:lp-ivanov}, we have
\begin{align*}
 \tilde{h}(\vy)=
\begin{cases}
 0 & (\textrm{if $\sum_{m=1}^{M}y_m^{-p}\leq 1$}),\\
-\infty & (\textrm{otherwise}).
\end{cases}
\end{align*}
Then the block-norm-based regularizer $g$ \eqref{eq:gdef} is defined as
 the minimum of the following constrained minimization problem
\begin{align*}
 g(\vx)=\frac{1}{2}\min_{\vy\in\Real_+^{M}}\vx\T\vy\quad {\rm s.t.}\quad
 \sum_{m=1}^{M}y_m^{-p}\leq 1.
\end{align*}
We define the Lagrangian
$\mathcal{L}=\frac{1}{2}\vx\T\vy+\frac{\eta}{2p}(\sum_{m=1}^{M}y_m^{-p}-1)$. Taking
the derivative of the Lagrangian and setting it to zero, we have
\begin{align*}
 y_m&=\left(\frac{\eta}{x_m}\right)^{1/(1+p)}.
\end{align*}
In addition, the multiplier $\eta$ is obtained as
 $\eta=\left(\sum_{m=1}^{M}x_m^{p/(1+p)}\right)^{(1+p)/p}$. Combining
 the above two expressions, we have
\begin{align*}
 g(\vx)&=\frac{1}{2}\left(\sum_{m=1}^{M}x_m^{p/(1+p)}\right)^{(1+p)/p},
\end{align*}
from which we obtain \Eqref{eq:mkl-blockq-2}.
\end{proof}

Except for the special structure inside the loss term, the multi-task
 learning problem is a special case of the $\ell_p$-norm MKL with
 $p=1$. Therefore, we have the following corollary.
\begin{corollary}[Multi-task learning]
Applying Theorem~\ref{thm:forward} to the multi-task learning problem in
 Example~\ref{ex:multi-task}, we have
\begin{align}
\label{eq:block-multi-task}
\minimize_{
\substack{
f_1\in\calH_1,\ldots,f_M\in\calH_M,\\
b\in\Real,\\
\vd\in\Real_+^{M}}}&\sum_{i=1}^N\ell\left(y_i, f_{l(i)}(x_i)+f_{M}(x_i)+b\right)
+\frac{C}{2}\left(\sum_{m=1}^M\|f_m\|_{\calH_m}\right)^2.
\end{align}
Moreover, the one-parameter case considered in \cite{EvgPon04,EvgMicPon05}, we have
\begin{align*}
\minimize_{
\substack{
f_1\in\calH_1,\ldots,f_M\in\calH_M,\\
b\in\Real,\\
\vd\in\Real_+^{M}}}&\sum_{i=1}^N\ell\left(y_i, f_{l(i)}(x_i)+f_{M}(x_i)+b\right)
+\frac{C}{2}\left(\sqrt{\textstyle\frac{1}{n}\sum_{m=1}^{n}\|f_m\|_{\calH_m}^2}+\|f_M\|_{\calH_M}\right)^2.
\end{align*}
\end{corollary}
\begin{proof}
 The first part is a special case of
 Corollary~\ref{cor:blockq-ivanov}. The one-parameter case can be
 derived using Jensen's inequality as follows. Let $d_m=n\lambda$
 ($m=1,\ldots,n$) and $d_M=1-\lambda$. Using Jensen's inequality, we
 have
\begin{align*}
 \frac{\frac{1}{n}\sum_{m=1}^{n}\|f_m\|_{\calH_m}^2}{\lambda}+\frac{\|f_M\|_{\calH_M}^2}{1-\lambda}
&=\lambda\left(\frac{\sqrt{\frac{1}{n}\sum_{m=1}^{n}\|f_m\|_{\calH_m}^2}}{\lambda}\right)^2+(1-\lambda)\left(\frac{\|f_M\|_{\calH_M}}{1-\lambda}\right)^2\\
&\geq \left(\sqrt{\textstyle\frac{1}{n}\sum_{m=1}^{n}\|f_m\|_{\calH_m}^2}+\|f_M\|_{\calH_M}\right)^2.
\end{align*}
The equality holds when $(1-\lambda)/\lambda=\|f_m\|_{\calH_M}/\sqrt{\frac{1}{n}\sum_{m=1}^{n}\|f_m\|_{\calH_m}^2}$.
\end{proof}
Note that in the one-parameter case discussed in
\cite{EvgPon04,EvgMicPon05}, the solution of the joint optimization of
task similarity $d_m$ and the classifier results in either
$\sqrt{\frac{1}{n}\sum_{m=1}^n\|f_m\|_{\calH_m}^2}=0$ (all the tasks are
the same) or $\|f_M\|_{\calH_M}=0$ (all the tasks are different). For
the general problem~\eqref{eq:block-multi-task}, some
$\|f_m\|_{\calH_m}$ may become zero but not necessarily all tasks.

\begin{corollary}[Wedge penalty]
Applying Theorem~\ref{thm:forward}  to the Wedge regularization in
 Example~\ref{ex:wedge}, we have
\begin{align}
\label{eq:mkl-block-wedge}
 \minimize_{f_1\in\calH_1,\ldots,
 f_M\in\calH_M,b\in\Real}&\sum_{i=1}^N\ell\left(y_i,\textstyle\sum_{m=1}^Mf_m(x_i)+b\right)
+Cg(\|f_1\|_{\calH_1}^2,\ldots,\|f_M\|_{\calH_M}^2).
\intertext{where}
\label{eq:gdef-wedge}
g(x_1,\ldots,x_M)&=\sup_{\eta_1,\ldots,\eta_{M-1}\geq 0}\sum_{m=1}^{M}\sqrt{(1+\eta_{m-1}-\eta_{m})x_m},
\end{align}
with $\eta_0=\eta_{M}=0$.
\end{corollary}
\begin{proof}
For the regularizer defined in Example~\ref{ex:wedge}, we have
\begin{align*}
g(\vx)&=\frac{1}{2}\min_{\vy\in\Real_+^{M}}\sum_{m=1}^{M}\left(x_my_m+y_m^{-1}\right)
 \quad {\rm s.t.} \quad y_{m+1}^{-1}\leq y_{m}^{-1} (m=1,\ldots,M-1).
\end{align*}
We define the Lagrangian $\mathcal{L}$ as follows:
\begin{align*}
 \mathcal{L}&=\frac{1}{2}\left(\sum_{m=1}^{M}\left(x_my_m+y_m^{-1}\right)+\sum_{m=1}^{M-1}\eta_m(y_{m+1}^{-1}-y_{m}^{-1})\right).
\end{align*}
Minimizing the Lagrangian with respect to $\vy$, we have
\begin{align*}
 y_m&=\sqrt{\frac{1+\eta_{m-1}-\eta_{m}}{x_m}},
\end{align*}
where we define $\eta_0=\eta_M=0$ for convenience. Substituting the
above expression back into the Lagrangian and forming the dual problem
we obtain \Eqref{eq:gdef-wedge}.
\end{proof}

The following theorem is useful in mapping algorithms defined in the
block-norm formulation (\Secref{sec:block-norm}) back to the
kernel-weight-based formulation (\Secref{sec:kernel-weight}).
\begin{theorem}[Converse of Theorem~\ref{thm:forward}]
\label{thm:backward}
If the block-norm-based regularizer $g$ in
 formulation~\eqref{eq:mkl-block-gen} is a concave
function, we can derive the corresponding kernel-weight based
regularizer $h$ as follows:
\begin{align*}
 h(d_1,\ldots,d_M)&=-2g^{\ast}\left(1/(2d_1),\ldots,1/(2d_M)\right),
\end{align*}
where $g^{\ast}$ denotes the concave conjugate of $g$ defined as
 follows:
\begin{align*}
 g^{\ast}(\vy)&=\inf_{\vx\in\Real_+^M}\left(\vx\T\vy-g(\vx)\right).
\end{align*}
\end{theorem}

\begin{proof}
 The converse statement
holds because
\begin{align*}
 h(d_1,\ldots,d_M)&=-\tilde{h}(1/d_1,\ldots,1/d_M)\\
&=-(2 g)^\ast(1/d_1,\ldots,1/d_M)\\
&=-2 g^\ast(1/(2d_1),\ldots,1/(2d_M)),
\end{align*}
where $g^{\ast}$ is the concave conjugate of $g$.
\end{proof}

Theorem~\ref{thm:backward} can be used to derive the kernel-weight-based
regularizer $h$ corresponding to the elastic-net MKL in
Example~\ref{ex:elastic} as in the following corollary.

\begin{corollary}[Kernel-weight-based formulation for Elastic-net MKL]
The block-norm-based regularizer $g$ defined in Example~\ref{ex:elastic}
 is a concave function, and the corresponding kernel-weight-based
 regularizer can be obtained as follows:
\begin{align*}
 h(\vd)&=\sum_{m=1}^{M}\frac{(1-\lambda)^2d_m}{1-\lambda d_m}.
\end{align*}
\end{corollary}
\begin{proof}
 Since the regularizer $g$ defined in Example~\ref{ex:elastic} is a
 convex combination of two concave functions (square-root and linear
 functions), it is clearly a concave function. Next, noticing that the
 regularizer $g$ is separable into each component, we 
obtain its concave conjugate as follows:
\begin{align}
\label{eq:hdef-elastic}
 \tilde{h}(y_m)&=\inf_{x_m\geq 0}\left(x_my_m-2\left((1-\lambda)\sqrt{x_m}+\frac{\lambda}{2}x_m\right)\right).
\end{align}
Taking the derivative with respect to $x_m$, we have
\begin{align*}
 y_m-\frac{1-\lambda}{\sqrt{x_m}}-\lambda&=0.
\end{align*}
Substituting the above expression back into \Eqref{eq:hdef-elastic}, we
 have
\begin{align*}
 \tilde{h}(y_m)&=-\frac{(1-\lambda)^2}{y_m-\lambda}.
\end{align*}
Therefore
\begin{align*}
 h(d_1,\ldots,d_M)&=-\sum_{m=1}^{M}\tilde{h}(1/d_m)=\sum_{m=1}^{M}\frac{(1-\lambda)^2d_m}{1-\lambda d_m}.
\end{align*}
\end{proof}
Note that in the special case $\lambda=0$ (corresponding to Example~\ref{ex:block1}), the kernel-weight-based
regularizer \eqref{eq:hdef-elastic} reduces to the
$\ell_p$-norm MKL with $p=1$ in Example~\ref{ex:lp-tikhonov}. See also Table~\ref{tab:correspond}.
\begin{table}[tb]
 \begin{center}
\caption{Correspondence between the regularizer $h$ in the
  kernel-weight-based regularization~\eqref{eq:mkl-tikhonov} and the
  concave function $g$ in the block-norm formulation \eqref{eq:mkl-block-gen}. $I_{[0,1]}$ denotes the indicator function of
  the interval $[0,1]$; i.e., $I_{[0,1]}(x)=0$ (if $x\in[0,1]$), and
  $I_{[0,1]}(x)=\infty$ (otherwise). }
\label{tab:correspond}
  \begin{tabular}{c|c|c|c|c}
   MKL model         & $g(x)$       & $h(d_m)$    & Optimal kernel weight\\
\hline
\hline
   block 1-norm MKL  & $\sqrt{x}$  & $d_m$   &$d_m=\|f_m\|_{\calH_m}$\\
\hline
   $\ell_p$-norm MKL & $\frac{1+p}{2p}x^{p/(1+p)}$ & $d_m^p/p$ & $d_m=\|f_m\|_{\calH_m}^{2/(1+p)}$\\
\hline
Uniform-weight MKL & \multirow{2}{*}{$x/2$} &
	   \multirow{2}{*}{$I_{[0,1]}(d_m)$} &   \multirow{2}{*}{$d_m=1$}\\
(block 2-norm MKL) &  & &\\
\hline
block  $q$-norm MKL &  \multirow{2}{*}{$\frac{1}{q}x^{q/2}$} &
	   \multirow{2}{*}{$-\frac{q-2}{q} d_m^{-q/(q-2)}$} & \multirow{2}{*}{$d_m=\|f_m\|_{\calH_m}^{2-q}$}\\
($q>2$) &   &  & \\
\hline
   Elastic-net MKL & $(1-\lambda)\sqrt{x}+\frac{\lambda}{2}x$ & 
$\frac{(1-\lambda)^2d_m}{1-\lambda d_m}$ &  $d_m=\frac{\|f_m\|_{\calH_m}}{(1-\lambda)+\lambda\|f_m\|_{\calH_m}}$
  \end{tabular}
 \end{center}
\end{table}

\section{Empirical Bayesian multiple kernel learning}
\label{sec:bayes}
For a separable regularizer (see Corollary~\ref{cor:separable}), the kernel-weight-based formulation~\eqref{eq:mkl-tikhonov} allows a
probabilistic interpretation as a hierarchical maximum a posteriori
(MAP) estimation problem as follows:
\begin{align}
\label{eq:mkl-tikhonov-sep}
  \minimize_{
\substack{
f_1\in\calH_1,\ldots,f_M\in\calH_M,\\
b\in\Real,\\
\vd\in\Real_+^{M}}}&\sum_{i=1}^N\ell\left(y_i,\textstyle\sum_{m=1}^Mf_m(x_i)+b\right)
+\frac{C}{2}\sum_{m=1}^M\left(\frac{\|f_m\|_{\calH_m}^2}{d_m}+  h(d_m)\right).
\end{align}
The loss term can be considered as a
negative log-likelihood. The first regularization term
$\|f_m\|_{\calH_m}^2/d_m$ can be considered as the negative log of a
Gaussian process prior with variance scaled by the hyper-parameter
$d_m$. The last regularization term $h(d_m)$ corresponds to the
negative log of a hyper-prior distribution $p(d_m)\propto \exp(- h(d_m))$.
In this section, instead of a MAP estimation, we propose to maximize the
marginal likelihood (evidence) to obtain the kernel weights; we 
call this approach empirical Bayesian MKL. 

We rewrite the (separable) kernel-weight regularization problem~\eqref{eq:mkl-tikhonov-sep} as a probabilistic generative model as
follows: 
\begin{align*}
d_m &\sim \frac{1}{Z_1}\exp(- h(d_m))\qquad(m=1,\ldots,M),\\
f_m & \sim GP(f_m;0, d_m k_m)\qquad(m=1,\ldots,M)\\
y_i &\sim \frac{1}{Z_2}\exp(-\ell(y_i,f_1(x_i)+f_2(x_i)+\cdots +f_M(x_i))),
\end{align*}
where $Z_1$ and $Z_2$ are
normalization constants; $GP(f;0,k)$ denotes the Gaussian
process~\citep{RasWil06} with mean zero and covariance function $k$. We
omit the bias term for simplicity.

When the loss function is quadratic
$\ell(y_i,z_i)=(y_i-z_i)^2/({2\sigma_y^2})$,
we can analytically integrate out the Gaussian process random variable
$(f_m)_{m=1}^M$  and compute the
negative log of the marginal likelihood as follows:
\begin{align}
\label{eq:nloglike}
-\log p(\vy|\vd)&=
\frac{1}{2}\vy\T\bar{\mK}(\vd)^{-1}\vy+\frac{1}{2}\log\left|\bar{\mK}(\vd)\right|
\intertext{where $\vd=(d_1,\ldots,d_M)\T$,
 $\mK_m=(k_m(x_i,x_j))_{i,j=1}^N$ is the Gram matrix, and}
\bar{\mK}(\vd)&:=\sigma_y^2\mI_N+\sum_{m=1}^Md_m\mK_m.\nonumber
\end{align}

Using the quadratic upper-bounding technique in \cite{WipNag08}, we can
rewrite the marginal likelihood~\eqref{eq:nloglike}
in a more meaningful expression as follows:
\begin{align}
\label{eq:obj-bayes}
-\log p(\vy|\vd)&=\!\!\!\min_{\substack{\vf_1\in\Real^N,\\\ldots,
 \vf_M\in\Real^N}}\!\!\left(
 \frac{1}{2\sigma_y^2}\left\|\vy-\textstyle\sum\limits_{m=1}^M\vf_m\right\|^2\!\!+\frac{1}{2}{\textstyle\sum\limits_{m=1}^M}\frac{\|\vf_m\|_{\mK_m}^2}{d_m}\right)
+ \frac{1}{2}\log\left|\sigma_y^2\mI_N+{\textstyle\sum\limits_{m=1}^M}d_m\mK_m\right|,
\end{align}
where $\vf_m:=(f_m(x_1),\ldots,f_m(x_N))\T$, and
$\|\vf_m\|_{\mK_m}^2=\vf_m\T\mK_m^{-1}\vf_m$. See
\cite{WipNag08,WipNag09} for more details. Now we can see that the minimization of marginal
likelihood~\eqref{eq:obj-bayes} is a special case of kernel-weight-based
regularization~\eqref{eq:mkl-tikhonov} with a {\em concave,
nonseparable} regularizer defined as
\begin{align*}
 h(\vd)&= \log\left|\sigma_y^2\mI_N+{\textstyle\sum\limits_{m=1}^M}d_m\mK_m\right|.
\end{align*}
We ask whether we can derive the block-norm formulation
(see \Secref{sec:block-norm}) for the marginal likelihood based
MKL. Unfortunately Theorem~\ref{thm:forward} is not applicable because $h$
is not convex. The result in \citet[Appendix B]{WipNag09} suggests that
reparameterizing $d_m=\exp(\eta_m)$, we have the following expression
for the block-norm-based regularizer:
\begin{align}
\label{eq:gdef-bayes}
 g(\vx)=\frac{1}{2}\inf_{\veta\in\Real^{M}}\left(\sum_{m=1}^{M}x_me^{-\eta_m}+\log\left|\sigma_y^2\mI_N+{\textstyle\sum\limits_{m=1}^{M}}e^{\eta_m}\mK_m\right|\right).
\end{align}
Importantly, the above minimization is a convex minimization
problem. However, this does not change the fact that the original
minimization problem~\eqref{eq:obj-bayes} is a non-convex one. In fact,
the term $\|\vf_m\|_{\mK_m}^2e^{-\eta_m}$ is convex for $\eta_m$
but not jointly convex for $\vf_m$ and $\eta_m$.

Instead of directly minimizing (e.g., by gradient descent) the marginal
likelihood~\eqref{eq:nloglike} to obtain a hyperparameter maximum
likelihood estimation, we employ an alternative approach known as the MacKay
update~\citep{Mac92,WipNag09}.
The MacKay update alternates between the minimization of right-hand-side of
\Eqref{eq:obj-bayes} with respect to $\vf_m$ and an (approximate)
minimization of \Eqref{eq:gdef-bayes} with respect to $\eta_m$.
The minimization in \Eqref{eq:obj-bayes} is a fixed kernel-weight learning
problem (see \Secref{sec:fixed}) and for the special case of quadratic
loss, it is simply a least-squares problem. For the approximate minimization
of \Eqref{eq:gdef-bayes} with respect to $\eta_m$, we simply take the
derivative of \Eqref{eq:gdef-bayes} and set is to zero as follows:
\begin{align*}
-x_m e^{-\eta_m}
+{\rm Tr}\left((\sigma^2\mI_N+\textstyle\sum_{m=1}^Me^{\eta_m}\mK_m)^{-1} e^{\eta_m}\mK_m\right)=0,
\end{align*}
from which we have the following update equation:
\begin{align*}
 e^{\eta_m}&\leftarrow\frac{x_m}{{\rm Tr}\left((\sigma^2\mI_N+\textstyle\sum_{m=1}^Me^{\eta_m}\mK_m)^{-1} e^{\eta_m}\mK_m\right)}.
\end{align*}
Therefore, substituting $d_m=e^{\eta_m}$ and $x_m=\|\vf_m\|_{\mK_m}^2$,
we obtain the following iteration:
\begin{align}
\label{eq:mackay-1}
 (\vf_m)_{m=1}^M&\leftarrow\argmin_{(\vf_m)_{m=1}^M}\Biggl(
 \frac{1}{2\sigma_y^2}\left\|\vy-\sum\nolimits_{m=1}^M\vf_m\right\|^2
+\frac{1}{2}\sum\nolimits_{m=1}^M\frac{\|\vf_m\|_{\mK_m}^2}{d_m}\Biggr)\\
\label{eq:mackay-2}
d_m &\leftarrow\frac{\|\vf_m\|_{\mK_m}^2}{{\rm
 Tr}\left((\sigma^2\mI_N+\sum_{m=1}^Md_m\mK_m)^{-1} d_m\mK_m\right)} \quad(m=1,\ldots,M).
\end{align}
The convergence of this procedure is not established theoretically, but it
is known to converge rapidly in many practical
situations~\citep{Tip01}. 

\section{Numerical experiments}
\label{sec:exp}
In this section, we numerically compare MKL algorithms we discussed in
this paper on three binary classification problems we have taken
from Caltech 101 dataset~\citep{FeiFerPer04}. We have generate 1,760
kernel functions by combining four SIFT features, 22 spacial
decompositions (including the spatial pyramid kernel), two kernel
functions, and 10 kernel parameters\footnote{Preprocessed data is
available from
\texttt{http://www.ibis.t.u-tokyo.ac.jp/ryotat/prmu09/data/}.}. More
precisely, the kernel functions were constructed as combinations of the
following  four factors in the prepossessing pipeline:
\begin{itemize}
 \item Four types of SIFT features, namely hsvsift (adaptive scale),
       sift (adaptive scale), sift (scale fixed to 4px), sift (scale
       fixed to 8px). We used the implementation by \cite{vdSGevSno10}. The local features were sampled uniformly
       (grid) from each input image. We randomly chose 200 local
       features and assigned visual words to every local features using
       these 200 points as cluster centers.
 \item Local histograms obtained by partitioning the image into
       rectangular cells of the same size in a hierarchical manner;
       i.e., level-0 partitioning has 1 cell (whole image) level-1
       partitioning has 4 cells and level-2 partitioning has 16
       cells. From each cell we computed a kernel function by
       measuring the similarity of the two local feature histograms
       computed in the same cell from two images. In addition, the
       spatial-pyramid kernel~\citep{GraDar07,LazSchPon06}, which
       combines these kernels by exponentially decaying weights, was
       computed. In total, we used 22 kernels (=one level-0 kernel + four
       level-1 kernels + 16 level-2 kernels + one spatial-pyramid
       kernel). See also \cite{GehNow09} for a similar approach. 
\item  Two kernel functions (similarity measures). We used the Gaussian
       kernel:
\begin{align*}
 k(q(x),q(x'))&=\exp\Bigl(-\sum_{j=1}^n\frac{(q_j(x)-q_j(x'))^2}{2\gamma^2}\Bigr),
\intertext{for 10 band-width parameters ($\gamma$'s) linearly spaced between
 $0.1$ and $5$ and the $\chi^2$-kernel:}
 k(q(x),q(x'))&=\exp\Bigl(-\gamma^2\sum_{j=1}^n\frac{(q_j(x)-q_j(x'))^2}{(q_j(x)+q_j(x'))}\Bigr)
\end{align*}
for 10 band-width parameters ($\gamma$'s) linearly spaced between $0.1$ and $10$,
       where $q(x),q(x')\in\mathbb{N}_+^n$ are the 
       histograms computed in some region of two images $x$ and $x'$.
\end{itemize}
The combination of 4 sift features, 22 spacial regions, 2
kernel functions, and 10 parameters resulted in  1,760 kernel functions
in total.

We compare uniform kernel combination, block 1-norm MKL, Elastic-net MKL
with $\lambda=0.5$, Elastic-net MKL with $\lambda$ chosen by cross
validation on the training-set, $\ell_2$-norm
MKL~\cite{UAI:Cortes+etal:2009,KloBreSonLasMueZie09} and empirical Bayesian
MKL. $\ell_2$-norm MKL uses the hinge loss, and the other MKL models
(except empirical Bayesian MKL) use the logit loss.  We also include uniform kernel combination with the squared loss to
make the comparison between the empirical Bayesian MKL and the
rest possible. Since the difference between Uniform (logit) and
Uniform (square) is small, we expect that the discussion here is not
specific to the choice of loss functions. 
For the Elastic-net MKL~\eqref{eq:mkl-block-elastic}, we either fix the
constant $\lambda$ as $\lambda=0.5$ (Elastic (0.5)) or we choose the
value of $\lambda$ from $\{0,0.2,\ldots,0.8,1\}$. MKL models with the logit
loss are implemented in SpicyMKL\footnote{Available from \texttt{http://www.simplex.t.u-tokyo.ac.jp/\~{}s-taiji/software/SpicyMKL/}.} toolbox~\citep{SuzTom09}.
For the empirical
Bayesian MKL, we use the MacKay update~\eqref{eq:mackay-1}-\eqref{eq:mackay-2}.
We used the implementation of $\ell_2$-norm MKL in Shogun toolbox~\cite{SonRatHenWidBehZieBonBinGehFra10}.
The regularization constant $C$ was chosen by $2\times 4$-fold cross
validation on the training-set for each method. We used the candidate
$\{0.0001, 0.001, 0.01, 0.1, 1, 10\}$ for all methods except
$\ell_2$-norm MKL and $\{0.01, 0.1, 1, 10, 100, 1000\}$ for the
$\ell_2$-norm MKL. 

\label{sec:exp}
\begin{figure}[tb]
 \begin{center}
  \subfigure[Accuracy averaged over 20 train/test splits]{\includegraphics[width=.5\textwidth]{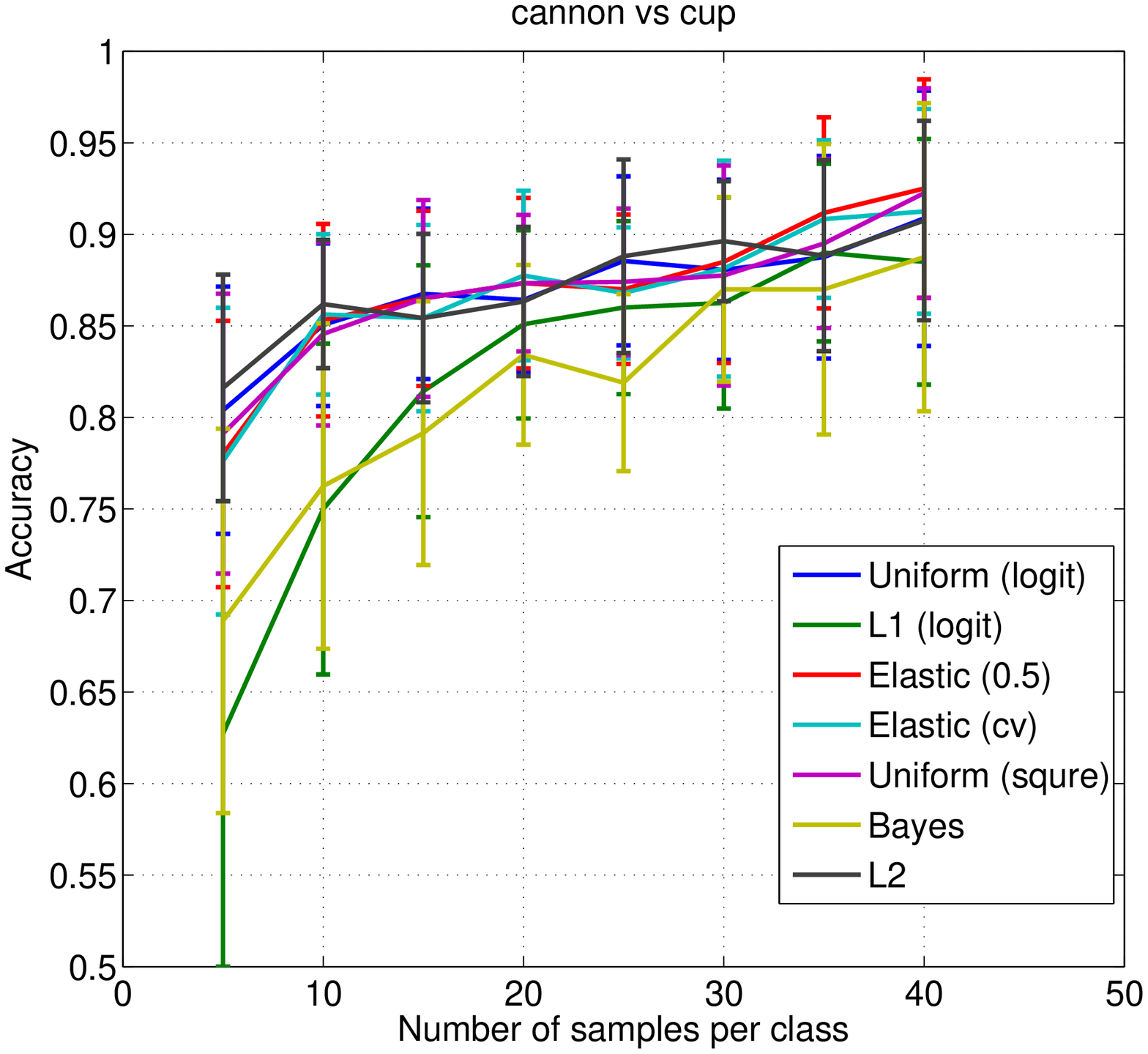}\label{fig:canon-cup-accuracy}}~\subfigure[Obtained
  kernel weights]{\includegraphics[width=.55\textwidth]{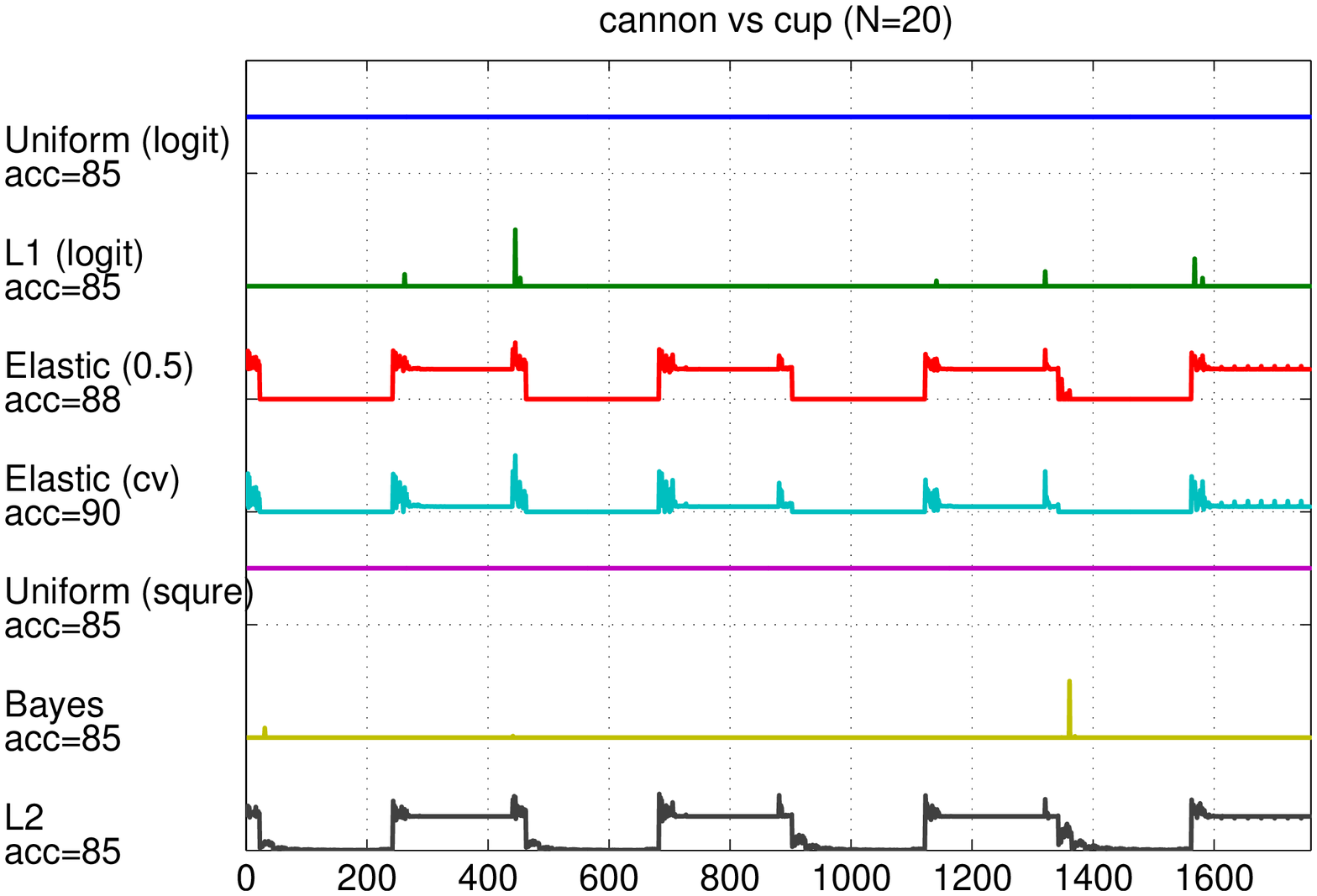}\label{fig:cannon-cup-weights}}
  \caption{Canon vs Cup from Caltech 101 dataset.}
  \label{fig:cannon-cup}
\end{center}
\end{figure}

\begin{figure}[tb]
 \begin{center}
  \subfigure[Accuracy averaged over 20 train/test splits]{\includegraphics[width=.5\textwidth]{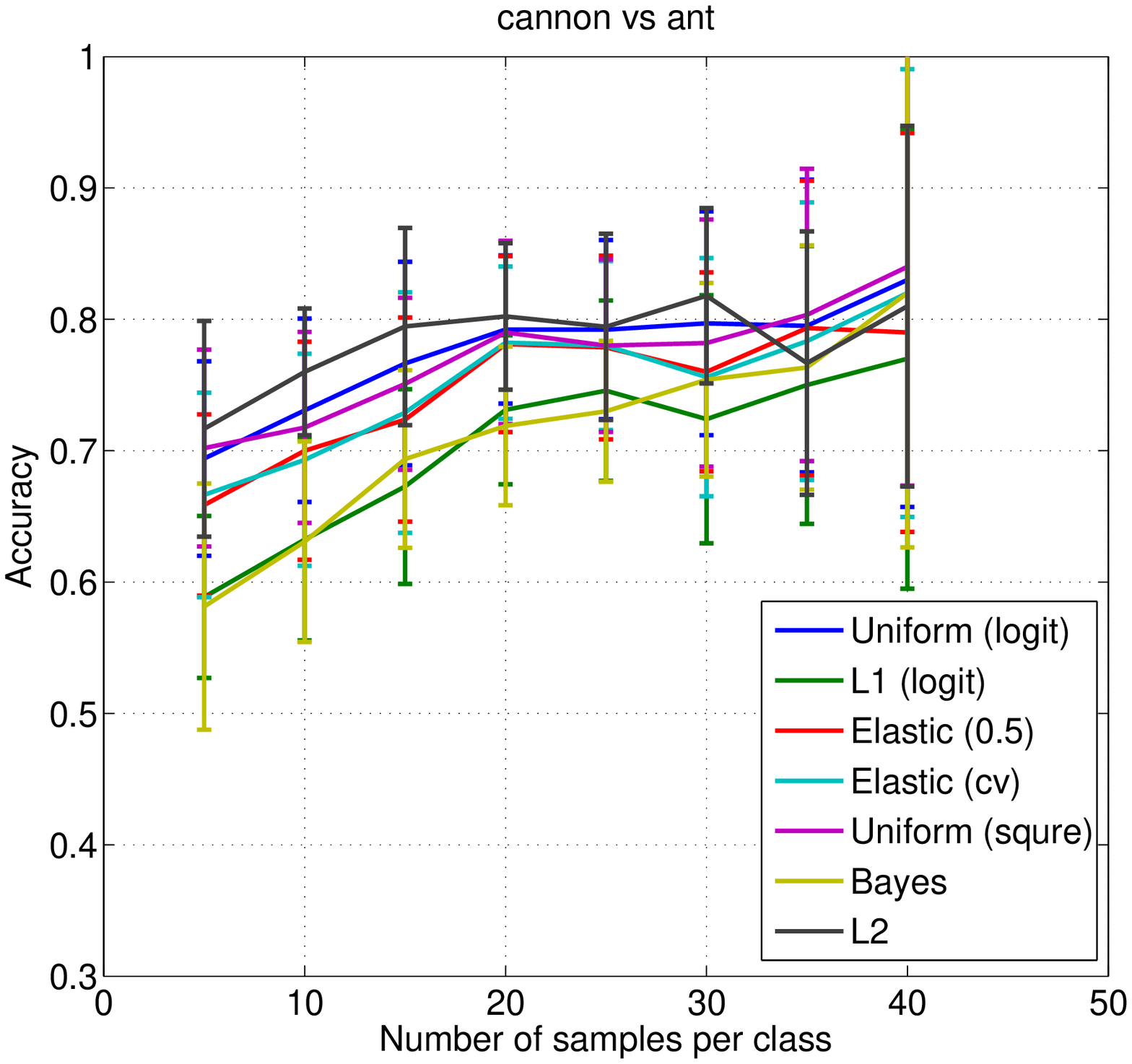}\label{fig:canon-ant-accuracy}}~\subfigure[Obtained
  kernel weights]{\includegraphics[width=.55\textwidth]{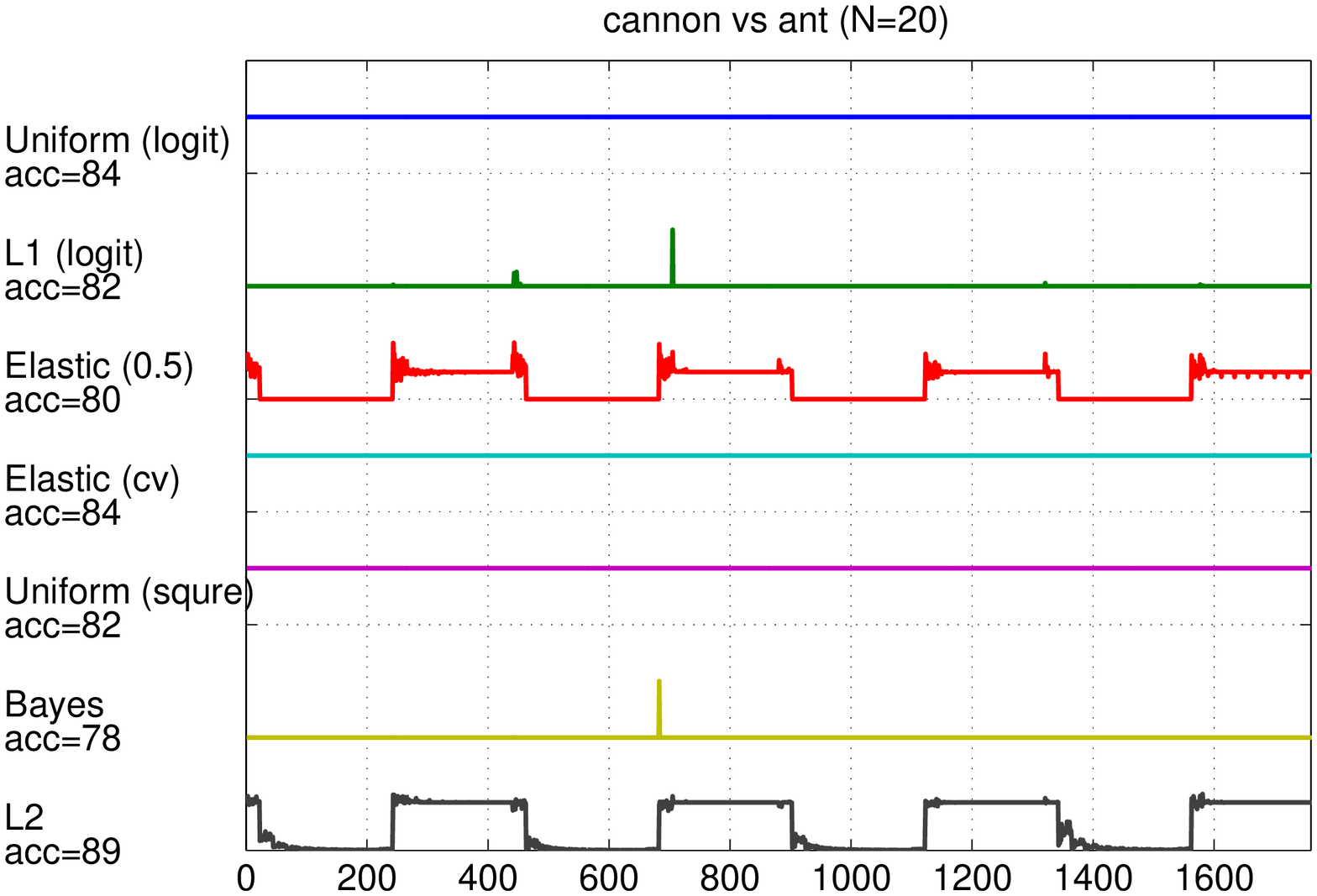}\label{fig:cannon-ant-weights}}
  \caption{Cannon vs Ant from Caltech 101 dataset.}
  \label{fig:cannon-ant}
\end{center}
\end{figure}

∑\begin{figure}[tb]
 \begin{center}
  \subfigure[Accuracy averaged over 20 train/test splits]{\includegraphics[width=.5\textwidth]{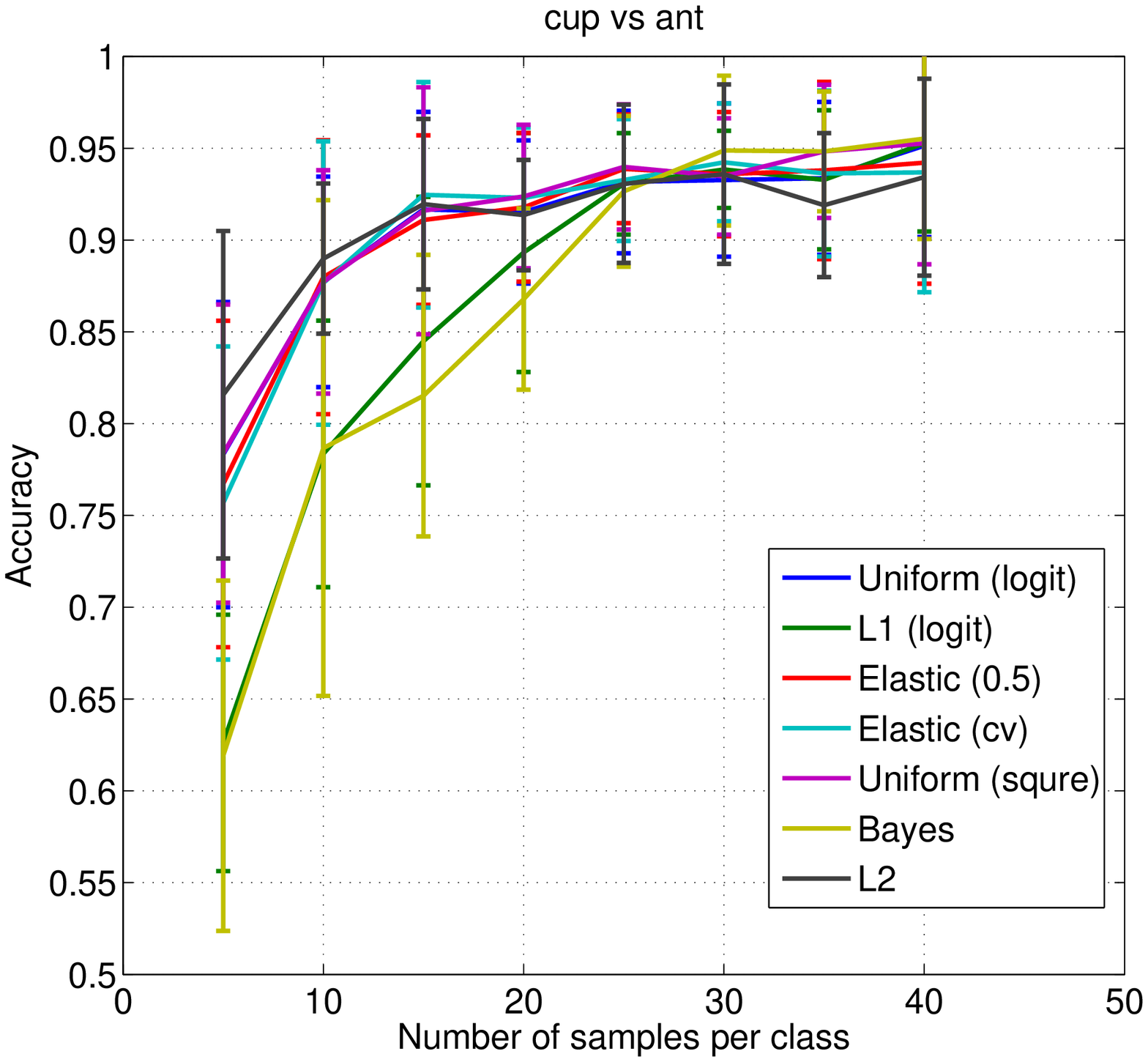}\label{fig:cup-ant-accuracy}}~\subfigure[Obtained
  kernel weights]{\includegraphics[width=.55\textwidth]{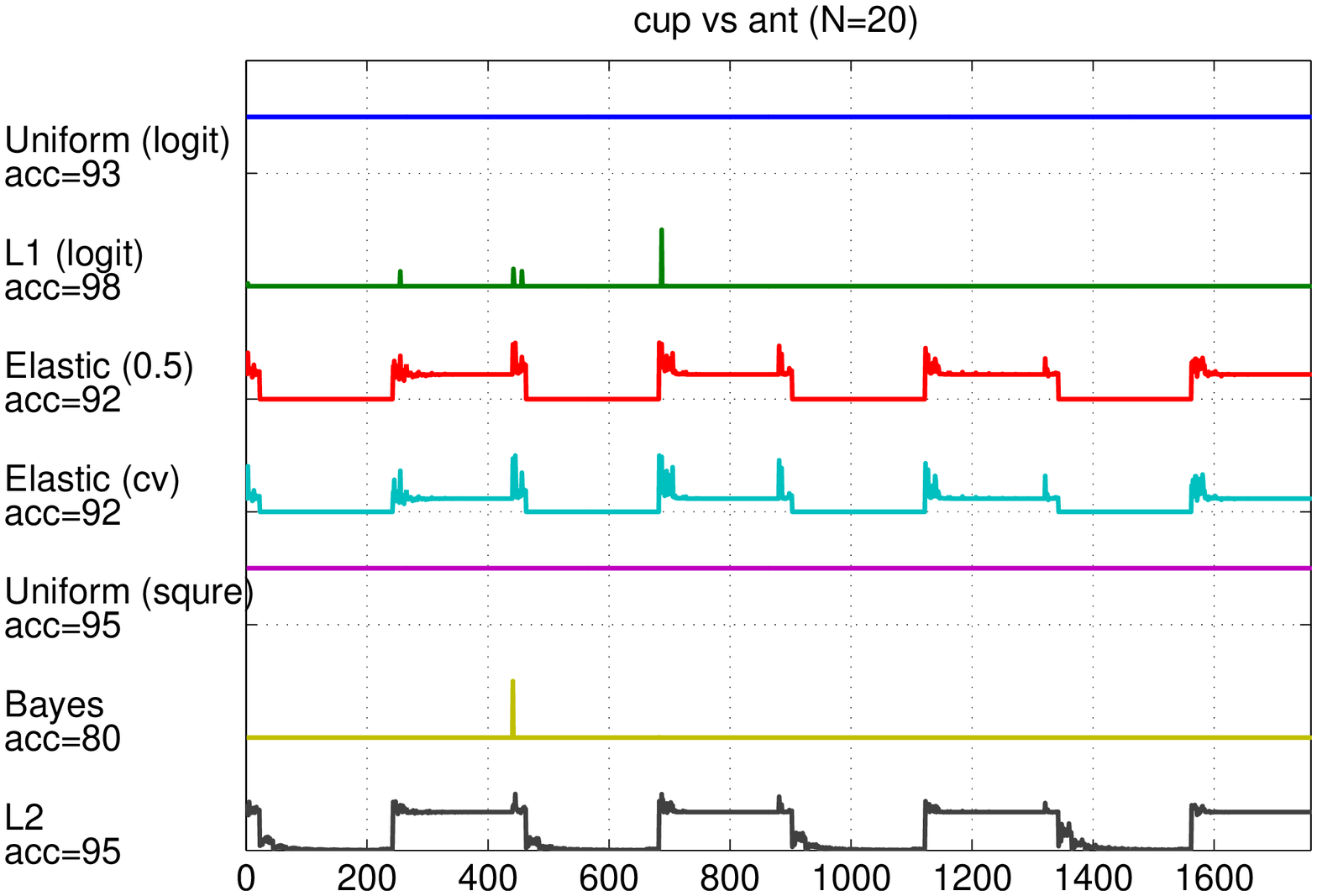}\label{fig:cup-ant-weights}}
  \caption{Cup vs Ant from Caltech 101 dataset.}
  \label{fig:cup-ant}
\end{center}
\end{figure}

Figures~\ref{fig:cannon-cup}--\ref{fig:cup-ant} show the
results of applying different MKL algorithms.
We can see that overall $\ell_2$-norm MKL, Elastic-net MKL,  and
uniformly-weighted MKL 
perform favorable compared to other MKL methods. Empirical Bayesian MKL
and block 1-norm MKL tend to perform worse than the above three methods,
especially when the number of
samples per class is smaller than 20. However, in \Figref{fig:cup-ant}, they
do perform comparably for the number of samples per class above 30.
Although Elastic-net MKL and $\ell_2$-norm MKL perform
almost the same as uniform MKL in terms of accuracy,
the right panels show that these methods can find important kernel
components automatically. More specifically, on the ``Cannon vs Cup''
dataset (\Figref{fig:cannon-cup}), Elastic-net MKL chose 88
Gaussian RBF kernel functions and 792 $\chi^2$ kernel functions. Thus it
prefers $\chi^2$ kernels to Gaussian RBF kernels. This agrees with the
common choice in computer vision literature. In addition, Elastic-net MKL
consistently chose the band width parameter $\gamma=0.1$ for the Gaussian
RBF kernels but it never chose $\gamma=0.1$ for the $\chi^2$ kernels;
instead it averaged all $\chi^2$ kernels from $\gamma=1.2$ to $\gamma=10$.

\section{Conclusion}
\label{sec:conclusion}
We have shown that various MKL and structured sparsity models including
$\ell_p$-norm MKL, Elastic-net MKL, multi-task learning, Wedge penalty,
and overlapped group lasso can be seen as applications of different
regularization strategies. These models have been conventionally
presented in either kernel-weight-based regularization or
block-norm-based regularization. We have shown that these two
formulations can be systematically mapped to each other under some
conditions through a concave conjugate transformation; see
Table~\ref{tab:correspond}. 

Furthermore, we have proposed a marginal-likelihood-based kernel
learning algorithm. We have shown that the propose empirical Bayesian
MKL can be considered to be employing a nonconvex nonseparable
regularizer on the kernel weights. Furthermore, we have derived the
expression for the block-norm regularizer corresponding to the proposed
empirical Bayesian MKL. 

We have tested the classification performance as well as the resulting
kernel weights of various regularized MKL models we have discussed on
visual categorization task from Caltech 101 dataset using 1,760 kernels.
We have shown that Elastic-net MKL can achieve comparable classification
accuracy to uniform kernel combination with roughly half of the
candidate kernels and provide information about the usefulness of the
candidate kernels. $\ell_2$-norm MKL also achieves similar
classification performance and qualitatively similar kernel
weights. However $\ell_2$-norm MKL does not achieve sparsity in the
kernel weights in contrast to Elastic-net MKL. Empirical Bayesian MKL tends to
perform worse than the above two methods probably because the kernel
weights it obtains becomes extremely sparse. One way to avoid such
solution is to introduce hyper-priors for the kernel weights as in \cite{Urt10}.

We are currently aiming to relax the sufficient condition in
Theorem~\ref{thm:forward} to guarantee mapping from the
kernel-weight-based formulation to block-norm-based formulation. We
would also like to have a finer characterization of the block-norm
regularizer corresponding to the empirical Bayesian MKL (see also
\cite{WipNag08}).  Theoretical argument concerning when to use sparse MKL
models (e.g., $\ell_1$-norm MKL or empirical Bayesian MKL) and when to
use non-sparse MKL models ($\ell_p$-norm MKL) is also necessary.

\subsection*{Acknowledgement}

We would like to thank Hisashi Kashima and Shinichi Nakajima for helpful
discussions. This work was partially supported by MEXT Kakenhi 22700138,
22700289, and NTT Communication Science Laboratories. 

{\small 
\bibliography{mkl}
}

\appendix
%

\section{Proof of \Eqref{eq:mkl-norm} in a finite dimensional case}
\label{sec:proof-mkl-norm}
In this section, we provide a proof of \Eqref{eq:mkl-norm} when
$\calH_1,\ldots,\calH_m$ are all finite dimensional. We assume
that the input space $\calX$ consists of $N$ points $x_1,\ldots,x_N$,
for example the training points. The function $f_m\in\calH_m$ is
completely specified by the function values at the $N$-points
$\vf_m=(f_m(x_1),\ldots,f_m(x_N))\T$. The kernel function $k_m$ is also
specified by the Gram matrix $\mK_m=(k_m(x_i,x_j))_{i,j=1}^{N}$. The
inner product $\dot{f_m}{g_m}_{\calH_m}$ is written as
$\dot{f_m}{g_m}_{\calH_m}=\vf_m\T\mK_m^{-1}\vg_m$, where $\vg_m$ is the
$N$-dimensional vector representation of $g_m\in\calH_m$, assuming that
the Gram matrix $\mK_m$ is positive definite. It is easy to check the
reproducibility; in fact,
$\dot{f_m}{k_m(\cdot,x_i)}=\vf_m\T\mK_m^{-1}\mK_m(:,i)=f(x_i)$, where
$\mK_m(:,i)$ is a column vector of the Gram matrix $\mK_m$ that corresponds to the
$i$th sample point $x_i$.

The right-hand side of \Eqref{eq:mkl-norm} is written as follows:
\begin{align*}
 \min_{\vf_1,\ldots,\vf_M\in\Real^N}&\sum_{m=1}^M\frac{\vf_m\T\mK_m^{-1}\vf_m}{d_m}\quad{\rm s.t.}\quad\sum_{m=1}^M\vf_m=\bar{\vf}.
\end{align*}
Forming the Lagrangian, we have
\begin{align*}
&\sum_{m=1}^M\frac{\vf_m\T\mK_m^{-1}\vf_m}{d_m}\\
&=\sum_{m=1}^M\frac{\vf_m\T\mK_m^{-1}\vf_m}{d_m}+2\valpha\T\!\!\left(\bar{\vf}-\sum\nolimits_{m=1}^M\vf_m\right)\\
&\geq-\valpha\T\left(\sum\nolimits_{m=1}^Md_m\mK_m\right)\valpha+2\valpha\T\bar{\vf}\\
&\xrightarrow{\max_{\valpha}} \bar{\vf}\T\left(\sum\nolimits_{m=1}^Md_m\mK_m\right)^{-1}\bar{\vf},
\end{align*}
where the equality is obtained for
$$\vf_m=d_m\mK_m\left(\sum\nolimits_{m=1}^Md_m\mK_m\right)^{-1}\bar{\vf}.$$

\end{document}